\documentclass{article}

\PassOptionsToPackage{numbers}{natbib}

\usepackage[final]{neurips_2019}

\usepackage[utf8]{inputenc} 
\usepackage[T1]{fontenc}    
\usepackage{hyperref}       
\usepackage{url}            
\usepackage{booktabs}       
\usepackage{amsfonts}       
\usepackage{nicefrac}       
\usepackage{microtype}      
\usepackage[eulergreek]{sansmath}
\usepackage{amsmath, amssymb, amsthm, enumerate, bm, capt-of,mathtools}
\usepackage[capitalise]{cleveref}

\usepackage{tikz, tikzscale,pgfplots}
\usetikzlibrary{arrows,shapes,backgrounds,patterns,fadings,matrix,arrows,calc,
	intersections,decorations.markings,
	positioning,external,arrows.meta}
\pgfplotsset{width = 0.48\textwidth, compat = 1.13, height=5.8cm, grid=major, 
	ticklabel style = {font=\sansmath\sffamily\scriptsize},
	ylabel style = {font = \sf\footnotesize, yshift = -0.2cm},
	xlabel style = {font = \sf\footnotesize, yshift = 0.15cm},
	zlabel style = {font = \sf\footnotesize},
	legend style = {font=\sf\scriptsize}, legend cell align = left, 
	title style={yshift=-7pt, font =\sf\footnotesize} }

\newtheorem{assumption}{Assumption}[section]
\newtheorem{remark}{Remark}[section]
\newtheorem{definition}{Definition}[section]
\newtheorem{lemma}{Lemma}[section]
\newtheorem{theorem}{Theorem}[section]

\DeclarePairedDelimiter{\norm}{\lVert}{\rVert}
\DeclarePairedDelimiter{\abs}{\lvert}{\rvert}

\newcommand{\T}{^{\intercal}}

\newcommand{\Bset}{\mathbb{B}}

\newcommand{\Nset}{\mathbb{N}}
\newcommand{\Rset}{\mathbb{R}}
\newcommand{\Xset}{\mathbb{X}}

\newcommand{\Uset}{\mathbb{U}}

\newcommand{\alli}{i=1,\ldots,N}

\newcommand{\N}{\mathcal{N}}
\newcommand{\fh}{\hat{f}}
\newcommand{\fhcd}{\hat{f}(\cdot)}

\newcommand{\e}{\bm{e}}
\newcommand{\kc}{k_c}
\newcommand{\sigon}{\sigma_n^2}
\newcommand{\x}{\bm{x}}

\makeatletter
\newcommand{\subalign}[1]{%
	\vcenter{%
		\Let@ \restore@math@cr \default@tag
		\baselineskip\fontdimen10 \scriptfont\tw@
		\advance\baselineskip\fontdimen12 \scriptfont\tw@
		\lineskip\thr@@\fontdimen8 \scriptfont\thr@@
		\lineskiplimit\lineskip
		\ialign{\hfil$\m@th\scriptstyle##$&$\m@th\scriptstyle{}##$\crcr
			#1\crcr
		}%
	}
}
\makeatother

\title{Uniform Error Bounds for Gaussian Process Regression with Application to Safe Control}

\author{%
	Armin Lederer \\
	Technical University of Munich\\
	\texttt{armin.lederer@tum.de}
	\And
	Jonas Umlauft \\
	Technical University of Munich\\
	\texttt{jonas.umlauft@tum.de}
	\And
	Sandra Hirche\\
	Technical University of Munich\\
	\texttt{hirche@tum.de}
}

\begin{document}
	
	\maketitle
	
	\begin{abstract}
		Data-driven models are subject to model errors due to limited and noisy 
		training data. Key to the application of such models in safety-critical 
		domains is the quantification of their model error. Gaussian processes 
		provide such a measure and uniform error bounds have been derived, which 
		allow safe control based on these models. However, existing error bounds 
		require restrictive assumptions. In this paper, we employ the Gaussian 
		process distribution and continuity arguments to derive a novel uniform 
		error bound under weaker assumptions. Furthermore, we demonstrate how this 
		distribution can be used to derive probabilistic Lipschitz constants and 
		analyze the asymptotic behavior of our bound. Finally, we derive safety 
		conditions for the control of unknown dynamical systems based on Gaussian 
		process models and evaluate them in simulations
		of a robotic manipulator.\looseness=-1
	\end{abstract}
	
	\section{Introduction}
	The application of machine learning techniques in control tasks bears 
	significant promises. The identification of highly nonlinear systems through 
	supervised learning techniques~\cite{Norgard2000} and the automated 
	policy search in reinforcement learning~\cite{Deisenroth2011a} enables the 
	control of complex unknown systems. Nevertheless, the application in 
	safety-critical domains, like autonomous driving, robotics or aviation is rare. 
	Even though the data-efficiency and performance of 
	self-learning controllers is impressive, engineers still hesitate to rely on 
	learning approaches if the physical integrity of systems is at risk, in 
	particular, if humans are involved. Empirical 
	evaluations, e.g. for autonomous driving~\cite{Huval2015}, are 
	available, however, this might not be sufficient to reach the desired level of 
	reliability and autonomy.
	
	Limited and noisy training data lead to imperfections in 
	data-driven models~\cite{Umlauft2017b}. This makes the quantification of the uncertainty in 
	the 
	model and the knowledge about a model's ignorance key for the utilization of 
	learning approaches in safety-critical applications. Gaussian process models 
	provide this measure for their own imprecision and therefore gained 
	attention in the control 
	community~\cite{Beckers2019,Berkenkamp2016a,Fanger2016}. 
	These approaches heavily rely on error bounds of Gaussian process regression 
	and are therefore limited by the strict assumptions made in previous works on 
	GP uniform error bounds~\cite{Srinivas2012,Chowdhury2017a,Umlauft2018,Umlauft2018a}.\looseness=-1
	
	The main contribution of this paper is therefore the derivation of a novel 
	GP uniform error bound, which requires less prior knowledge and 
	assumptions than previous approaches and is therefore applicable to a wider 
	range of problems. Furthermore, we derive a Lipschitz constant for the samples 
	of GPs and investigate the asymptotic behavior in order to demonstrate that
	arbitrarily small error bounds can be guaranteed with sufficient computational
	resources and data. The proposed GP bounds are employed to derive safety 
	guarantees for unknown dynamical systems which are controlled based on a GP 
	model. By employing Lyapunov theory~\cite{Khalil2002}, we prove that the 
	closed-loop system - here we take a robotic manipulator as example - converges 
	to a small fraction of the state space and can therefore be considered as 
	safe.\looseness=-1
	
	The remainder of this paper is structured as follows: We briefly introduce 
	Gaussian process regression and discuss related error bounds in 
	\cref{sec:background}. The novel proposed GP uniform error bound, the probabilistic Lipschitz 
	constant and the asymptotic analysis are presented in \cref{sec:errorbound}. In 
	\cref{sec:safety} we show safety of a GP model based controller and  
	evaluate it on a robotic manipulator in \cref{sec:numeval}.
	
	\section{Background}
	\label{sec:background}
	
	\subsection{Gaussian Process Regression and Uniform Error Bounds}
	
	Gaussian process regression is a Bayesian machine learning method based
	on the assumption that any finite collection of random variables\footnote{Notation: 
		Lower/upper case bold symbols denote vectors/matrices and 
		$\Rset_{+}$/$\Rset_{+,0}$ all real positive/non-negative numbers. 
		$\Nset$ denotes all natural numbers,
		$\bm{I}_n$ the $n\times n$ identity matrix, 
		the dot in~$\dot{x}$ the derivative of~$x$ with respect to time and 
		$\|\cdot\|$ the Euclidean norm. A function $f(\bm{x})$ is said to admit a 
		modulus of continuity $\omega:\Rset_+\rightarrow\Rset_+$ 
		if and only if $|f(\bm{x})-f(\bm{x}')|\leq\omega(\|\bm{x}-\bm{x}'\|)$. 
		The $\tau$-covering number $M(\tau,\Xset)$ of a set 
		$\Xset$ (with respect to the Euclidean metric) is defined as 
		the minimum number of spherical balls with radius $\tau$ which is 
		required to completely cover $\Xset$. Big $\mathcal{O}$ notation is 
		used to describe the asymptotic behavior of functions.
	} 
	$y_i\in\Rset$
	follows a joint Gaussian distribution with prior mean~$0$ and covariance 
	kernel $k:\Rset^d\times\Rset^d\rightarrow\Rset_+$ \cite{Rasmussen2006}. 
	Therefore, the variables $y_i$ are observations of a sample function 
	$f:\Xset\subset\Rset^d\rightarrow\Rset$ of the GP distribution 
	perturbed by zero mean Gaussian noise with variance $\sigon\in \Rset_{+,0}$. By 
	concatenating $N$ input data points $\x_i$ in a matrix $\bm{X}_N$ the 
	elements of the GP kernel matrix $\bm{K}(\bm{X}_N,\bm{X}_N)$ are defined as 
	$K_{ij}=k(\x_i,\x_j)$,~$i,j=1,\ldots,N$ and $\bm{k}(\bm{X}_N,\x)$ denotes the 
	kernel vector, which is defined 
	accordingly. 
	The probability distribution of the GP at a point $\x$ conditioned on the 
	training data concatenated in $\bm{X}_N$ and $\bm{y}_N$ is then given as a normal 
	distribution with mean $\nu_N(\x)=\bm{k}(\x,\bm{X}_N)
	(\bm{K}(\bm{X}_N,\bm{X}_N)+\sigon\bm{I}_N)^{-1}\bm{y}_N$ and 
	variance $\sigma_N^2(\x,\x')=k(\x,\x')-\bm{k}(\x,\bm{X}_N)
	(\bm{K}(\bm{X}_N,\bm{X}_N)+\sigon\bm{I}_N)^{-1}\bm{k}(\bm{X}_N,\x')$.
	
	A major reason for the popularity of GPs and related approaches in safety critical
	applications is the existence of uniform error bounds for the regression error, 
	which is defined as follows.
	\begin{definition}
		Gaussian process regression exhibits a uniformly bounded error 
		on a compact set~$\Xset\subset\mathbb{R}^d$ if there exists 
		a function $\eta(\bm{x})$ such that 
		\begin{align}
		|\nu_N(\bm{x})-f(\bm{x})|\leq \eta(\bm{x})\quad\forall\bm{x}\in\Xset.
		\end{align}
		If this bound holds with probability of at least $1-\delta$ for some 
		$\delta\in (0,1)$, it is called a probabilistic uniform error bound.
	\end{definition}
	
	\subsection{Related Work}
	
	For many methods closely related to Gaussian process regression, uniform error 
	bounds are very common. When dealing with noise-free data, i.e. in 
	interpolation of multivariate functions, results from the field of scattered 
	data approximation with radial basis functions can be applied 
	\cite{Wendland2005}. In fact, many of the results from interpolation with 
	radial basis functions can be directly applied to noise-free GP regression with 
	stationary kernels. The classical result in \cite{Wu1993} employs Fourier 
	transform methods to derive an error bound for functions in the reproducing 
	kernel Hilbert space (RKHS) attached to the interpolation kernel. By 
	additionally exploiting properties of the RKHS a uniform error bound with 
	increased convergence rate is derived in \cite{Schaback2002}. Typically, this 
	form of bound crucially depends on the so called power function, 
	which corresponds to the posterior standard deviation of Gaussian process 
	regression under certain conditions \cite{Kanagawa2018}. In \cite{Hubbert2004}, 
	a $\mathcal{L}_p$ error bound for data distributed on a sphere is developed, 
	while the bound in \cite{Narcowich2006} extends existing approaches to 
	functions from Sobolev spaces. Bounds for anisotropic kernels and the 
	derivatives of the interpolant are developed in \cite{Beatson2010}. A Sobolev 
	type error bound for interpolation with Mat\'ern kernels is derived in 
	\cite{Stuart2018}. Moreover, it is shown that convergence of the interpolation 
	error implies convergence of the GP posterior variance.\looseness=-1
	
	Regularized kernel regression is a method which extends many ideas from 
	scattered data interpolation to noisy observations and it is highly related to 
	Gaussian process regression as pointed out in \cite{Kanagawa2018}. In fact, the 
	GP posterior mean function is identical to kernel ridge regression with squared 
	cost function~\cite{Rasmussen2006}. Many error bounds such as 
	\cite{Mendelson2002} depend on the empirical $\mathcal{L}_2$ covering number 
	and the norm of the unknown function in the RKHS attached to the regression 
	kernel. In \cite{Zhang2005}, the effective dimension of the feature space, in 
	which regression is performed, is employed to derive a probabilistic uniform 
	error bound. 
	The effect of approximations of the kernel, e.g. with the Nystr\"om method, on 
	the regression error is analyzed in \cite{Cortes2010}. Tight error bounds using 
	empirical $\mathcal{L}_2$ covering numbers are derived under mild assumptions 
	in \cite{Shi2013}. Finally, error bounds for general regularization are 
	developed in \cite{Dicker2017}, which depend on regularization and the RKHS 
	norm of the function. 
	
	Using similar RKHS-based methods for Gaussian process regression, probabilistic 
	uniform error bounds depending on the maximal information gain and the RKHS norm 
	have been developed in~\cite{Srinivas2012}. These constants pose a high hurdle 
	which has prevented the rigorous application of this work in control 
	and typically heuristic constants without theoretical foundations are applied, 
	see e.g.~\cite{Berkenkamp2017a}. While regularized kernel regression allows a 
	wide range of observation noise distributions, the bound in \cite{Srinivas2012} 
	only holds for bounded sub-Gaussian noise. Based on this work an improved bound 
	is derived in \cite{Chowdhury2017a} in order to analyze the regret of an upper 
	confidence bound algorithm in multi-armed bandit problems. Although these 
	bounds are frequently used in safe reinforcement learning and control, they 
	suffer from several issues. On the one hand, they depend on constants which are 
	very 
	difficult to calculate. While this is no problem for theoretical analysis, it 
	prohibits the integration of these bounds into algorithms and often estimates 
	of the constants must be used. On the other hand, they suffer from the general 
	problem of RKHS approaches: The space of functions, for which the bounds 
	hold, becomes smaller the smoother the kernel is \cite{Narcowich2006}. In fact, 
	the RKHS attached to a covariance kernel is usually small compared to the 
	support 
	of the prior distribution of a Gaussian process \cite{VanderVaart2011}.
	
	The latter issue has been addressed by considering the support of the prior 
	distribution of the Gaussian process as belief space. Based on bounds for the 
	suprema of GPs \cite{Adler2007} and existing error bounds for interpolation 
	with radial basis functions, a probabilistic uniform error bound for Kriging 
	(alternative term for GP regression for noise-free training data) is derived in 
	\cite{Wang2019}. However, the uniform error of Gaussian process regression with 
	noisy observations has not been analyzed with the help of the prior GP 
	distribution to the best of our knowledge.
	
	\section{Probabilistic Uniform Error Bound}
	\label{sec:errorbound}
	While probabilistic uniform error bounds for the cases of noise-free 
	observations 
	and the restriction to subspaces of a RKHS are widely used, they often rely on 
	constants which are hard to determine and are typically limited to 
	unnecessarily small function spaces. The inherent probability distribution of 
	GPs, which is the largest possible function space for regression with a certain 
	GP, has not been exploited to derive uniform error bounds for Gaussian process 
	regression with noisy observations. Under the weak assumption of Lipschitz 
	continuity of the covariance kernel and the unknown function, a directly 
	computable probabilistic uniform error bound is derived in \cref{subsec:regerror}. 
	We demonstrate how Lipschitz constants for unknown functions directly follow from 
	the assumed distribution over the function space in \cref{subsec:prob 
		Lipschitz}. Finally, we show that an arbitrarily small error bound can be 
	reached with sufficiently many and well-distributed training data in 
	\cref{subsec:asymptotics}.

	\subsection{Exploiting Lipschitz Continuity of the Unknown Function}
	\label{subsec:regerror}
	
	In contrast to the RKHS based approaches in \cite{Srinivas2012,Chowdhury2017a}, we make use of 
	the inherent probability distribution over the function 
	space defined by Gaussian processes. We achieve this 
	through the following assumption.\looseness=-1
	\begin{assumption}
		\label{ass:samplefun}
		The unknown function $f(\cdot)$ is a sample from a 
		Gaussian process $\mathcal{GP}(0,k(\bm{x},\bm{x}'))$ 
		and observations $y=f(\x)+\epsilon$ are perturbed 
		by zero mean i.i.d. Gaussian noise $\epsilon$ with variance 
		$\sigon$.
	\end{assumption}
	This assumption includes abundant information about the regression problem. The 
	space of sample functions $\mathcal{F}$ is limited through the choice of the 
	kernel $k(\cdot,\cdot)$ of the Gaussian process. Using Mercer's decomposition 
	\cite{Mercer1909} $\phi_i(\bm{x})$, $i=1,\ldots,\infty$ of the kernel 
	$k(\cdot,\cdot)$, this space is defined through
	\begin{align}
	\mathcal{F}=\left\{ f(\x):~\exists \lambda_i,i=1,\ldots,\infty \text{ such that 
	} 
	f(\x)=\sum\limits_{i=1}^{\infty}\lambda_i\phi_i(\x) \right\},
	\end{align}
	which contains all functions that can be represented in terms of the kernel 
	$k(\cdot,\cdot)$. By choosing a suitable class of covariance functions 
	$k(\cdot,\cdot)$, this space can be designed in order to incorporate prior 
	knowledge of the unknown function $f(\cdot)$. For example, for covariance 
	kernels $k(\cdot,\cdot)$ which are universal in the sense of 
	\cite{Steinwart2001}, continuous functions can be learned with arbitrary 
	precision. Moreover, for the squared exponential kernel, the space of sample 
	functions corresponds to the space of continuous functions on $\Xset$, while 
	its RKHS is limited to analytic functions \cite{VanderVaart2011}. 
	Furthermore, \cref{ass:samplefun} defines a prior GP distribution over the 
	sample space $\mathcal{F}$ which is the basis for the calculation of the 
	posterior probability. The prior distribution is typically shaped by the 
	hyperparameters of the covariance kernel $k(\cdot,\cdot)$, e.g. slowly varying 
	functions can be assigned a higher probability than functions with high 
	derivatives. 
	Finally, \cref{ass:samplefun} allows Gaussian observation noise which is in contrast to the bounded noise required e.g. in \cite{Srinivas2012, Chowdhury2017a}.
	
	In addition to \cref{ass:samplefun}, we need Lipschitz continuity of the kernel 
	$k(\cdot,\cdot)$ and the unknown function $f(\cdot)$. We define the Lipschitz 
	constant of a differentiable covariance kernel $k(\cdot,\cdot)$ as
	\begin{align}
	L_k&\coloneqq\max\limits_{\x,\x'\in\Xset}\left\| \begin{bmatrix}
	\frac{\partial k(\x,\x')}{\partial x_1}&\ldots&
	\frac{\partial k(\x,\x')}{\partial x_d}
	\end{bmatrix}^T \right\|.
	\end{align}
	Since most of the practically used covariance kernels $k(\cdot,\cdot)$, such as 
	squared exponential and Mat{\'e}rn  kernels, are Lipschitz continuous 
	\cite{Rasmussen2006}, this is a weak restriction on covariance kernels. 
	However, it allows us to prove continuity of the posterior mean function 
	$\nu_N(\cdot)$ and the posterior standard deviation $\sigma_N(\cdot)$, which is 
	exploited to derive a probabilistic uniform error bound in the following theorem. 
	The proofs for all following theorems can be found in the supplementary 
	material.
	\begin{theorem}
		\label{th:errbound_with}
		Consider a zero mean Gaussian process defined through the continuous
		covariance kernel $k(\cdot,\cdot)$ with Lipschitz constant $L_k$ 
		on the compact set $\Xset$. Furthermore, consider a continuous unknown 
		function $f:\Xset\rightarrow \Rset$ with Lipschitz constant $L_f$ 
		and $N\in\mathbb{N}$ observations $y_i$ satisfying \cref{ass:samplefun}. Then, the 
		posterior mean function~$\nu_{N}(\cdot)$ and standard deviation 
		$\sigma_N(\cdot)$ 
		of a Gaussian process conditioned on the training data $\{(\bm{x}_i,y_i)\}_{i=1}^N$
		are continuous with Lipschitz constant $L_{\nu_N}$ and modulus of continuity
		$\omega_{\sigma_N}(\cdot) $ on $\Xset$ such that
		\begin{align}
		\label{eq:L_nu}
		L_{\nu_N}&\leq L_k\sqrt{N} 
		\left\| (\bm{K}(\bm{X}_N,\bm{X}_N)+\sigon\bm{I}_N)^{-1}\bm{y}_N \right\|\\
		\omega_{\sigma_N}(\tau)&\leq\sqrt{2\tau L_k\left(1+N
			\|(\bm{K}(\bm{X}_N,\bm{X}_N)+\sigon\bm{I}_N)^{-1}\|
			\max\limits_{\x,\x'\in\Xset}k(\x,\x')\right)}.
		\label{eq:omega_sigma}
		\end{align}
		Moreover, pick $\delta\in (0,1)$, $\tau\in\Rset_+$ and set 
		\begin{align}
		\label{eq:beta}
		\beta(\tau)&=2\log\left(\frac{M(\tau,\Xset)}{\delta}\right)\\
		\gamma(\tau)&=\left( L_{\nu_N}+L_f\right)\tau+\sqrt{\beta(\tau)}\omega_{\sigma_N}(\tau).
		\label{eq:gamma}
		\end{align}
		Then, it holds that
		\begin{align}
		\label{eq:errorbound}
		P\left(|f(\x)-\nu_{N}(\x)|\leq 
		\sqrt{\beta(\tau)}\sigma_{N}(\x)+\gamma(\tau), 
		~\forall\x\in\Xset\right)\geq 1-\delta.
		\end{align}
	\end{theorem}
	The parameter $\tau$ is in fact the grid constant of a grid used 
	in the derivation of the theorem. The error on the grid can be bounded
	by exploiting properties of the Gaussian distribution \cite{Srinivas2012} 
	resulting in a dependency on the number of grid points. Eventually, this leads to 
	the constant $\beta(\tau)$ defined in \eqref{eq:beta} since the covering number 
	$M(\tau,\Xset)$ 
	is the minimum number of points in a grid over $\Xset$ with grid constant 
	$\tau$. By employing the Lipschitz constant $L_{\nu_N}$ and the modulus of 
	continuity $\omega_{\sigma_N}(\cdot)$, which are trivially obtained due 
	Lipschitz continuity of the covariance kernel $k(\cdot,\cdot)$, as well as the 
	Lipschitz constant $L_f$, the error bound is extended to the complete set 
	$\Xset$, which results in \eqref{eq:errorbound}.
	
	Note, that most of the equations in \cref{th:errbound_with} can be directly 
	evaluated. Although our expression for~$\beta(\tau)$ depends on the covering 
	number of $\Xset$, which is in general difficult to calculate, upper bounds can 
	be computed trivially. For example, for a hypercubic set 
	$\Xset\subset\Rset^d$ the covering number can be bounded by
	\begin{align}
	M(\tau,\Xset)\leq \left(1+\frac{r}{\tau}\right)^d,
	\end{align}
	where $r$ is the edge length of the hypercube. Furthermore, 
	\eqref{eq:L_nu} and \eqref{eq:omega_sigma} depend only on the training data and 
	kernel expressions, which can be calculated analytically in general. Therefore, 
	\eqref{eq:errorbound} can be computed for fixed $\tau$ and $\delta$ if an upper 
	bound for the Lipschitz constant $L_f$ of the unknown function $f(\cdot)$ is 
	known. Prior bounds on the Lipschitz constant $L_f$ are often available for 
	control systems, e.g. based on simplified first order physical models. However, 
	we demonstrate a method to obtain probabilistic Lipschitz constants from 
	\cref{ass:samplefun} in \cref{subsec:prob Lipschitz}. Therefore, it is trivial 
	to compute all expressions in \cref{th:errbound_with} or upper bounds thereof, 
	which emphasizes the high applicability of \cref{th:errbound_with} in safe 
	control of unknown systems.
	
	Moreover, it should be noted that $\tau$ can be chosen arbitrarily small such 
	that the effect of the constant~$\gamma(\tau)$ can always be reduced to an 
	amount which is negligible compared to $\sqrt{\beta(\tau)}\sigma_N(\x)$. Even 
	conservative approximations of the Lipschitz constants $L_{\nu_N}$ and $L_f$ 
	and a loose modulus of continuity $\omega_{\sigma_N}(\tau)$ do not affect 
	the error bound \eqref{eq:errorbound} much since \eqref{eq:beta} grows merely 
	logarithmically with diminishing $\tau$. In fact, even the bounds \eqref{eq:L_nu}
	and \eqref{eq:omega_sigma}, which grow in the order of $\mathcal{O}(N)$
	and $\mathcal{O}(N^{\frac{1}{2}})$, respectively, as shown in the proof of 
	Theorem~\ref{thm:asym_behavior} and thus are unbounded, can be compensated such 
	that a vanishing uniform error bound can be proven under weak assumptions in 
	\cref{subsec:asymptotics}.\looseness=-1
	
	\subsection{Probabilistic Lipschitz Constants for Gaussian Processes}
	\label{subsec:prob Lipschitz}
	
	If little prior knowledge of the unknown function $f(\cdot)$ is given, it might 
	not be possible to directly derive a Lipschitz constant $L_f$ on $\Xset$. 
	However, we indirectly assume a certain distribution of the derivatives of 
	$f(\cdot)$ with \cref{ass:samplefun}. Therefore, it is possible to derive a 
	probabilistic Lipschitz constant $L_f$ from this assumption, which is described 
	in the following theorem.
	\begin{theorem}
		\label{th:errbound_without}
		Consider a zero mean Gaussian process defined through the
		covariance kernel $k(\cdot,\cdot)$ 
		with continuous partial derivatives up to the fourth order
		and partial derivative kernels 
		\begin{align}
		k^{\partial i}(\x,\x')&=\frac{\partial^2}{\partial x_i\partial x_i'}
		k(\x,\x')\quad \forall i=1,\ldots, d.
		\end{align}
		Let $L_k^{\partial i}$ denote the Lipschitz constants of the partial 
		derivative kernels $k^{\partial i}(\cdot,\cdot)$ on the set $\Xset$ 
		with maximal extension 
		\mbox{$r=\max_{\x,\x'\in\Xset}\|\x-\x'\|$}. 
		Then, a sample function $f(\cdot)$ of the Gaussian process is almost 
		surely continuous on $\Xset$ and with probability of at least $1-\delta_L$, 
		it holds that
		\begin{align}
		L_f=\left\|\begin{bmatrix}
		\sqrt{2\log\left( \frac{2d}{\delta_L} \right)}
		\max\limits_{\x\in\Xset}
		\sqrt{k^{\partial 1}(\x,\x)}+12\sqrt{6d}
		\max\left\{\max\limits_{\x\in\Xset}
		\sqrt{k^{\partial 1}(\x,\x)},\sqrt{rL_k^{\partial 1}}\right\}\\
		\vdots\\
		\sqrt{2\log\left( \frac{2d}{\delta_L} \right)}
		\max\limits_{\x\in\Xset}
		\sqrt{k^{\partial d}(\x,\x)}+12\sqrt{6d}
		\max\left\{\max\limits_{\x\in\Xset}
		\sqrt{k^{\partial d}(\x,\x)},\sqrt{rL_k^{\partial d}}\right\}
		\end{bmatrix}\right\|
		\label{eq:Lfbound}
		\end{align}
		is a Lipschitz constant of $f(\cdot)$ on $\Xset$.
	\end{theorem}
	Note that a higher differentiability of the covariance kernel $k(\cdot,\cdot)$ 
	is required compared to \cref{th:errbound_with}. The reason for this is that 
	the proof of \cref{th:errbound_without} exploits the fact that the partial 
	derivative $k^{\partial i}(\cdot,\cdot)$ of a differentiable kernel is again a 
	covariance function, which defines a derivative Gaussian 
	process~\cite{Ghosal2006}. In order to obtain continuity of the samples of 
	these 
	derivative processes, the derivative kernels $k^{\partial i}(\cdot,\cdot)$ must 
	be continuously differentiable \cite{Dudley1967}. Using the metric 
	entropy criterion \cite{Dudley1967} and the Borell-TIS inequality 
	\cite{Talagrand1994}, we exploit the continuity of sample functions and 
	bound their maximum value, which directly translates into the probabilistic
	Lipschitz constant \eqref{eq:Lfbound}.
	
	Note that all the values required in \eqref{eq:Lfbound} can be 
	directly computed. The maximum of the derivative kernels $k^{\partial 
		i}(\cdot,\cdot)$ as well as their Lipschitz constants $L_k^{\partial i}$ can be 
	calculated analytically for many kernels. Therefore, the Lipschitz constant 
	obtained with \cref{th:errbound_without} can be directly used in 
	\cref{th:errbound_with} through application of the union bound. Since the 
	Lipschitz constant $L_f$ has only a logarithmic dependence on the probability 
	$\delta_L$, small error probabilities for the Lipschitz constant can easily be 
	achieved.\looseness=-1
	\begin{remark}
		The work in~\cite{Gonzalez2016} derives also estimates for the 
		Lipschitz constants. However, they only take the Lipschitz constant of the 
		posterior mean function, which neglects the probabilistic nature of the GP 
		and thereby underestimates the Lipschitz constants of samples of the GP.
	\end{remark}
	\subsection{Analysis of Asymptotic Behavior}
	\label{subsec:asymptotics}
	
	In safe reinforcement learning and control of unknown systems an important 
	question regards the existence of lower bounds for the learning error because 
	they limit the achievable control performance. It is clear that the available 
	data and constraints on the computational resources pose such lower bounds in 
	practice. However, it is not clear under which conditions, e.g. requirements of 
	computational power, an arbitrarily low uniform error can be guaranteed. The 
	asymptotic analysis of the error bound, i.e. investigation of the bound 
	\eqref{eq:errorbound} in the limit $N\rightarrow\infty$ can clarify this 
	question. The following theorem is the result of this analysis.
	\begin{theorem}
		\label{thm:asym_behavior}
		Consider a zero mean Gaussian process defined through the continuous
		covariance kernel $k(\cdot,\cdot)$ with Lipschitz constant $L_k$ 
		on the set $\Xset$. Furthermore, consider an infinite data stream of 
		observations~$(\x_i,y_i)$ of an unknown function~\mbox{$f:\Xset\rightarrow 
			\Rset$} 
		with Lipschitz constant $L_f$ and maximum absolute value 
		$\bar{f}\in\Rset_{+}$ on 
		$\Xset$ which satisfies \cref{ass:samplefun}.
		Let $\nu_N(\cdot)$ and $\sigma_N(\cdot)$ denote the mean and standard deviation 
		of the Gaussian process conditioned on the first $N$ observations. If there 
		exists a~$\epsilon>0$ such that the standard deviation satisfies
		$\sigma_N(\x)\in\mathcal{O}\left(\log(N)^{-\frac{1}{2}-\epsilon}\right)$, 
		$\forall\x\in\Xset$,
		then it holds for every $\delta\in(0,1)$ that 
		\begin{align}
		P\left(\sup_{\bm{x}\in\Xset}\|\nu_N(\bm{x})-f(\bm{x})\|\in 
		\mathcal{O}(\log(N)^{-\epsilon})\right)\geq 1-\delta.
		\label{eq:as_bound}
		\end{align}
	\end{theorem}
	In addition to the conditions of \cref{th:errbound_with} the absolute value of 
	the unknown function is required to be bounded by a value $\bar{f}$. This is 
	necessary to bound the Lipschitz constant $L_{\nu_N}$ of the posterior mean 
	function $\nu_N(\cdot)$ in the limit of infinite training data. Even if no such 
	constant is known, it can be derived from properties of the GP under weak 
	conditions similarly to \cref{th:errbound_without}. Based on this restriction,
	it can be shown that the bound of the Lipschitz
	constant $L_{\nu_N}$ grows at most with rate~$\mathcal{O}(N)$ using the 
	triangle 
	inequality and the fact that the squared 
	norm of the observation noise~$\|\bm{\epsilon}\|^2$ follows a $\chi^2_N$ 
	distribution 
	with probabilistically bounded maximum value \cite{Laurent2000}.
	Therefore, we pick $\tau(N)\in\mathcal{O}(N^{-2})$ such that 
	$\gamma(\tau(N))\in\mathcal{O}(N^{-1})$ and $\beta(\tau(N))\in\mathcal{O}(\log(N))$ 
	which implies~\eqref{eq:as_bound}.

	The condition on the convergence rate of the posterior standard deviation 
	$\sigma_N(\cdot)$ in Theorem~\ref{thm:asym_behavior} can be seen as a condition 
	for the 
	distribution of the training data, which depends on the structure of the 
	covariance kernel. In~\cite[Corollary 3.2]{Lederer2019a}, the condition 
	is formulated as follows: Let~$\Bset_{\rho}(\x)$ denote a set of training points
	around~$\x$ with radius~$\rho>0$, then the posterior variance converges to zero 
	if there exists a function~$\rho(N)$ for which~$\rho(N)\leq k(\x,\x)/L_k~ 
	\forall N$,~$\lim_{N\to\infty}\rho(N)=0$ and~$\lim_{N\to\infty}\left| 
	\Bset_{\rho(N)}(\x) \right|=\infty$ holds.
	This is achieved, e.g. if a constant fraction of all samples lies on the 
	point~$\x$. 
	In fact, it is straightforward to derive a similar 
	condition for the uniform error bounds in~\cite{Srinivas2012,Chowdhury2017a}. 
	However, due to their dependence on the maximal information gain, the required 
	decrease rates depend on the covariance kernel $k(\cdot,\cdot)$ and are 
	typically higher. For example, the posterior standard deviation of a Gaussian 
	process 
	with a squared exponential kernel must 
	satisfy~$\sigma_N(\cdot)\in\mathcal{O}\left( \log(N)^{-\frac{d}{2}-2} \right)$ 
	for 
	\cite{Srinivas2012} and $\sigma_N(\cdot)\in\mathcal{O}\left( 
	\log(N)^{-\frac{d+1}{2}} \right)$ for \cite{Chowdhury2017a}.

	\section{Safety Guarantees for Control of Unknown Dynamical Systems}
	\label{sec:safety}
	Safety guarantees for dynamical systems, in terms of upper bounds for the 
	tracking error, are becoming more and more relevant as learning controllers are 
	applied in safety-critical applications like autonomous driving or robots 
	working in close proximity to 
	humans~\cite{Umlauft2017,Umlauft2017a,Umlauft2017b}. We therefore show how the 
	results in \cref{th:errbound_with} can be applied to control safely unknown 
	dynamical systems. In \cref{subsec:track} we propose a tracking control law for 
	systems which are learned with GPs. The stability 
	of the resulting controller is analyzed in \cref{subsec:stab}. 
	
	\subsection{Tracking Control Design}
	\label{subsec:track}
	Consider the nonlinear control affine dynamical system 
	\begin{align}
	\label{eq:sys}
	\dot{x}_1  = x_2,\qquad \dot{x}_2 = f(\x) + u,
	\end{align}
	with state~$\x=[x_1\ x_2]\T \in \Xset \subset \Rset^2$ and control 
	input~$u \in \Uset \subseteq \Rset$. While the structure of the 
	dynamics~\eqref{eq:sys} is known, the function~$f(\cdot)$ is not. However, we 
	assume that it is a sample from a GP with kernel~$k(\cdot,\cdot)$. Systems of 
	the 
	form~\eqref{eq:sys} cover a large range of applications including Lagrangian 
	dynamics and many physical systems. 
	
	The task is to define a policy~$\pi:\Xset \to \Uset$ for which the 
	output~$x_1$ tracks the desired trajectory~$x_d(t)$ such that the tracking 
	error~$\e =[e_1\  e_2]\T=\x-\x_d$ with~$\x_d =[x_d\ \dot{x}_d]\T$  vanishes 
	over 
	time, i.e.~$\lim_{t\to\infty} \norm{\e}= 0$. For notational simplicity, we 
	introduce the filtered state~$r =\lambda e_1 + e_2$, $\lambda\in \Rset_+$.
	
	A well-known method for tracking of control affine systems is 
	feedback linearization~\cite{Khalil2002}, which 
	aims for a 
	model-based compensation of the non-linearity~$f(\cdot)$ using an 
	estimate~$\hat{f}(\cdot)$ and then applies linear control principles for the 
	tracking. The feedback linearizing policy reads as
	\begin{align}
	\label{eq:FeliCtrl}
	u = \pi(\x) = -\fh(\x) +\nu,
	\end{align}
	where the linear control law~$\nu$ is the PD-controller
	\begin{align}
	\label{eq:LinCtrl}
	\nu=\ddot{x}_d  -\kc r - \lambda e_2, 
	\end{align}
	with control gain~$\kc\in \Rset_+$. This results in the dynamics of the 
	filtered state
	\begin{align}
	\dot{r}
	= f(\x)-\fh(\x) -\kc r.
	\end{align} 
	
	Assuming training data of the real system~$y_i = f(\x_i) + \epsilon$, $\alli$, 
	$\epsilon\sim\N(0,\sigon)$ are available, we utilize the posterior mean 
	function~$\nu_N(\cdot)$ for the model estimate~$\fhcd$.
	This implies, that observations of~$\dot{x}_2$ 
	are corrupted by noise, while~$\x$ is measured free of noise. This is of course 
	debatable, but in practice measuring the time derivative is usually realized 
	with finite difference approximations, which 
	injects significantly more noise than a direct measurement.
	
	\subsection{Stability Analysis}
	\label{subsec:stab}
	Due to safety constraints, e.g. for robots interacting with humans,  
	it is usually necessary to verify that the model~$\fhcd$ is sufficiently 
	precise and the parameters of the controller~$\kc,\lambda$ are chosen properly. 
	These safety 
	certificates can be achieved if there exists an upper bound for the tracking 
	error as defined in the following.\looseness=-1
	\begin{definition}[Ultimate Boundedness]
		\label{def:boundedness}
		The trajectory~$\x(t)$ of a dynamical system~$\dot{\x} = \bm{f}(\x,\bm{u})$ 
		is globally ultimately bounded, if there exist a positive constants~$b\in 
		\Rset_+$ such that for every~$a\in \Rset_+$, there is a~$T=T(a,b)\in 
		\Rset_+$ such that
		\begin{align*}
		\norm{\x(t_0)}\leq a \quad \Rightarrow \quad \norm{\x(t)}\leq b,\ 
		\forall t\geq t_0 + T.
		\end{align*}
	\end{definition}
	Since the solutions~$\x(t)$ cannot be computed analytically, a 
	stability analysis is necessary, which allows conclusions regarding the 
	closed-loop behavior without running the policy on the real 
	system~\cite{Khalil2002}.
	\begin{theorem}
		\label{thm:Stable}
		Consider a control affine system~\eqref{eq:sys}, where~$f(\cdot)$ 
		admits a Lipschitz constant $L_f$ on $\Xset\subset\Rset^d$. 
		Assume that~$f(\cdot)$ and the observations $y_i$, $i=1,\ldots,N$,  
		satisfy the conditions of \cref{ass:samplefun}. Then, the feedback 
		linearizing controller~\eqref{eq:FeliCtrl} with~$\fhcd = \nu_N(\cdot)$
		guarantees with probability~$1-\delta$ that the tracking error $\e$
		converges to 
		\begin{align}
		\Bset = \left\{\x\in \Xset \left| \norm{\e}
		\leq\frac{\sqrt{\beta(\tau)}\sigma_{N}(\x)+\gamma(\tau)}
		{\kc\sqrt{\lambda^2+1}}\right.\right\},
		\end{align}
		with $\beta(\tau)$ and $\gamma(\tau)$ defined in \cref{th:errbound_with}.
	\end{theorem}
	Based on Lyapunov theory, it can be shown that the tracking error converges if 
	the feedback term~$|\kc r|$ dominates the model error~$|f(\cdot)-\fh(\cdot)|$. 
	As 
	\cref{th:errbound_with} bounds the model error, the set for which holds~$|\kc 
	r|>\sqrt{\beta(\tau)}\sigma_{N}(\x)+\gamma(\tau)$ can be computed.
	It can directly be seen, that the ultimate bound can be made arbitrarily small, 
	by increasing the gains~$\lambda,\kc$ or with more training points to 
	decrease~$\sigma_{N}(\cdot)$.
	Computing the set~$\Bset$ allows to check whether the 
	controller~\eqref{eq:FeliCtrl} adheres to the safety 
	requirements.

	\section{Numerical Evaluation}
	\label{sec:numeval}
	We evaluate our theoretical results in two simulations.$\!$\footnote{Matlab 
		code 
		is online available:
		\url{https://gitlab.lrz.de/ga68car/GPerrorbounds4safecontrol}} In 
	\cref{subsec:num2D},
	we investigate the effect of applying \cref{th:errbound_without} to determine
	a probabilistic Lipschitz constant for an unknown synthetic system. Furthermore,
	we analyze the effect of unevenly distributed training samples on the tracking
	error bound from \cref{thm:Stable}. In \cref{subsec:robot}, we apply the 
	feedback
	linearizing controller \eqref{eq:FeliCtrl} to a tracking problem with a robotic
	manipulator. 
	
	\subsection{Synthetic System with Unknown Lipschitz Constant~$L_f$}
	\label{subsec:num2D}
	As an example for a system of form~\eqref{eq:sys}, we consider~$f(\x) = 
	1-\sin(x_1) +  \frac{1}{1+\exp(-x_2)}$. Based on a uniform grid over~$[0\ 
	3]\times[-3\ 3]$ the training set is formed of~$81$ points with~$\sigon = 
	0.01$. The reference trajectory is a circle~$x_d(t) = 2\sin(t)$ and the 
	controller gains are~$\kc=2$ and~$\lambda=1$. We choose a probability of 
	failure~$\delta = 0.01$,~$\delta_L = 0.01$ and set~$\tau = 10^{-8}$. The state 
	space is the rectangle~$\Xset =[-6\ 4]\times[-4\ 4]$. A squared exponential 
	kernel with automatic relevance determination is utilized, for which~$L_k$ 
	and~$\max_{\x,\x'\in\Xset}k(\x,\x')$ is derived analytically for the optimized 
	hyperparameters.
	We make use of \cref{th:errbound_without} to estimate the Lipschitz 
	constant~$L_f$, and it turns out to be a conservative bound (factor $10\sim 
	100$). However, this is not crucial, because~$\tau$ can be chosen arbitrarily 
	small and $\gamma(\tau)$ is dominated by~$\sqrt{\beta(\tau)}\omega_{\sigma_N}(\tau)$. As 
	\cref{th:errbound_with,th:errbound_without} are subsequently utilized in this 
	example, a union bound approximation can be applied to combine~$\delta$ 
	and~$\delta_L$.\looseness=-1
	
	The results are shown in \cref{fig:LyapDecrTe,fig:error}. Both plots show, that 
	the safety bound here is rather conservative, which also results from the fact 
	that the violation probability was set to~$1\%$.
	\looseness=-1

\begin{figure}
	\centering 
	\includegraphics{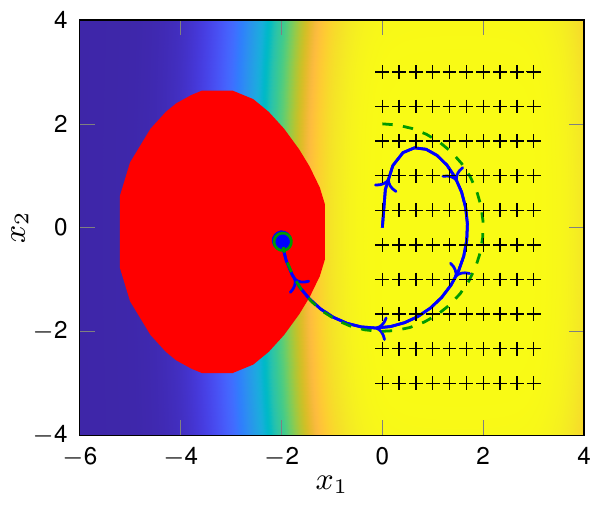}
	\includegraphics{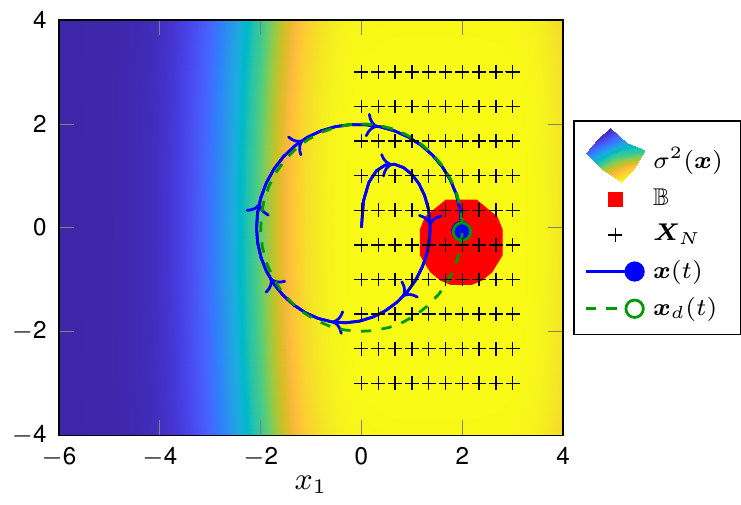}
	\caption{Snapshots of the state trajectory (blue) as it approaches the 
	desired trajectory (green). In low uncertainty areas (yellow background), 
	the set~$\Bset$ (red) is significantly smaller then in high uncertainty 
	areas (blue background).}
	\label{fig:LyapDecrTe}
\end{figure}

\begin{figure*}
	\centering
	\includegraphics{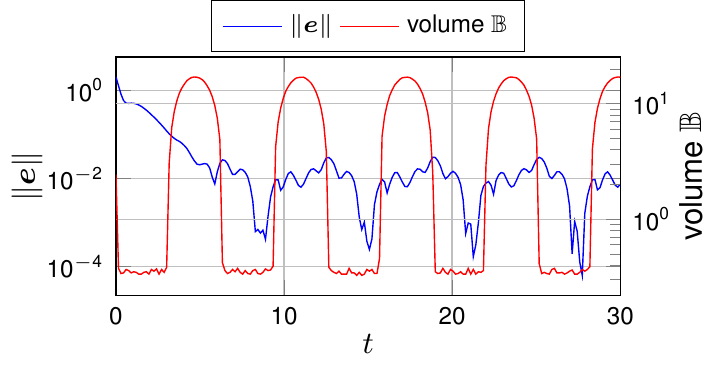}
	\caption{When the ultimate bound (red) is large, 
		the tracking error (blue) increases due to the less precise model.}
	\label{fig:error}
\end{figure*}

\subsection{Robotic Manipulator with 2 Degrees of Freedom}
\label{subsec:robot}
We consider a planar robotic manipulator in the~$z_1$-$z_2$-plane with 2 
degrees 
of freedom 
(DoFs), with unit length and unit masses / inertia for all links. For this 
example, we consider~$L_f$ to be known and extend 
\cref{th:errbound_with} to the multidimensional case using the union bound. 
The state space is here four dimensional~$[q_1\ \dot{q}_1\ q_2\ \dot{q}_2]$ and 
we consider $\Xset = [-\pi\ \pi]^4$. 
The $81$ training points are distributed in~$[-1\ 1]^4$ and the control gain 
is~$\kc=7$, while other constants remain the same as in \cref{subsec:num2D}. The 
desired trajectories for both joints are again sinusoidal as 
shown in \cref{fig:RobotLyap} on the right side. The robot dynamics are derived 
according to \cite[Chapter 4]{Murray1994}.

\Cref{th:errbound_with} allows to derive an error bound in the joint space of 
the robot according to \cref{thm:Stable}, which can be transformed into the task
space as shown in \cref{fig:RobotLyap} on the left. Thus, based on the learned 
(initially unknown) dynamics, it can be guaranteed, that the robot will not 
leave the depicted area and can thereby be considered as safe.

Previous error bounds for GPs are not applicable to this practical setting, 
because they i) do not allow the observation noise on the training data to be 
Gaussian~\cite{Srinivas2012}, which is a common assumption in robotics, ii) 
utilize constants which cannot be computed efficiently (e.g. maximal 
information 
gain in~\cite{Srinivas2010}) or iii) make assumptions difficult to 
verify in practice (e.g. the RKHS norm of the unknown dynamical 
system~\cite{Berkenkamp2016a}).
\begin{figure}
	\centering 
	\includegraphics{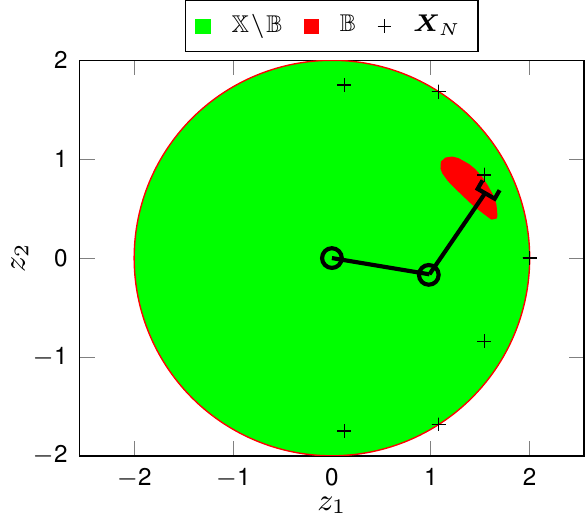}
	\includegraphics{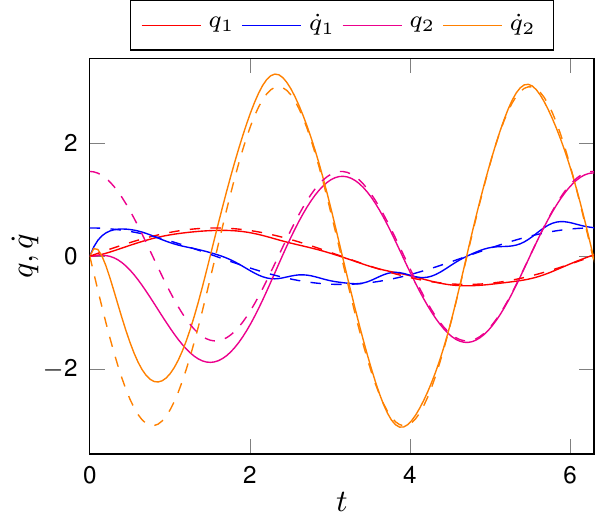}
	\caption{The task space of the robot (left) shows the robot is guaranteed 
		to remain in~$\Bset$ (red) after a transient phase. Hence, the remaining 
		state space~$\Xset\setminus\Bset$~(green) can be considered as safe. The 
		joint angles and velocities (right) converge to the desired trajectories 
		(dashed lines) over time.}
	\label{fig:RobotLyap}
\end{figure}

\section{Conclusion}
This paper presents a novel uniform error bound for Gaussian process regression. 
By exploiting the inherent probability distribution of Gaussian processes instead
of the reproducing kernel Hilbert space attached to the covariance kernel, a 
wider class of functions can be 
considered. Furthermore, we demonstrate how probabilistic Lipschitz constants can 
be estimated from the GP distribution and derive sufficient conditions to reach
arbitrarily small uniform error bounds. We employ the derived results to show safety
bounds for a tracking control algorithm and evaluate them in simulation for a 
robotic manipulator.

\subsubsection*{Acknowledgments}

Armin Lederer gratefully  acknowledges  financial  support from  the German Academic 
Scholarship Foundation.

\bibliographystyle{IEEEtran}
\bibliography{myBib}

\appendix

\section{Proof of Theorem 3.1}

\begin{proof}[Proof of Theorem 3.1]
	We first prove the Lipschitz constant of the posterior mean $\nu_N(\bm{x})$ and 
	the modulus of continuity of the standard deviation $\sigma_N(\bm{x})$, before
	we derive the bound of the regression error. The norm of the difference between 
	the posterior mean $\nu_N(\bm{x})$ evaluated at two different points is given by
	\begin{align*}
	\|\nu_N(\bm{x})-\nu_N(\bm{x}')\|&=
	\left\| \left(\bm{k}(\bm{x},\bm{X}_N)-\bm{k}(\bm{x}',\bm{X}_N)\right)
	\bm{\alpha}\right\|
	\end{align*}
	with 
	\begin{align}
	\bm{\alpha}=(\bm{K}(\bm{X}_N,\bm{X}_N)+\sigma_n^2\bm{I}_N)^{-1}\bm{y}_N.
	\end{align}
	Due to the Cauchy-Schwarz inequality and the Lipschitz continuity of the 
	kernel we obtain
	\begin{align*}
	\|\nu_N(\bm{x})-\nu_N(\bm{x}')\|&\leq L_k\sqrt{N} \left\| \bm{\alpha} 
	\right\|\|\bm{x}-\bm{x}'\|,
	\end{align*}
	which proves Lipschitz continuity of the mean $\nu_N(\bm{x})$.
	In order to calculate a modulus of continuity for the posterior standard deviation
	$\sigma_N(\bm{x})$ observe that the difference of the variance at two points 
	$\bm{x},\bm{x}'\in\mathbb{X}$ can be expressed as 
	\begin{align}
	|\sigma_N^2(\bm{x})-\sigma_N^2(\bm{x}')|&=|\sigma_N(\bm{x})-\sigma_N(\bm{x}')||
	\sigma_N(\bm{x})+\sigma_N(\bm{x}')|.
	\end{align}
	Since the standard deviation is positive semidefinite we have
	\begin{align}
	|\sigma_N(\bm{x})+\sigma_N(\bm{x}')|\geq |\sigma_N(\bm{x})-\sigma_N(\bm{x}')|
	\end{align}
	and hence, we obtain
	\begin{align}
	|\sigma_N^2(\bm{x})-\sigma_N^2(\bm{x}')|\geq |\sigma_N(\bm{x})-\sigma_N(\bm{x}')|^2.
	\end{align}
	Therefore, it is sufficient to bound the difference of the variance at two 
	points $\bm{x},\bm{x}'\in\mathbb{X}$ and take the square root of the resulting
	expression. Due to the Cauchy-Schwarz inequality and Lipschitz continuity of 
	$k(\cdot,\cdot)$ the absolute value of the 
	difference of the variance can be bounded by
	\begin{align}
	&|\sigma_N^2(\bm{x})-\sigma_N^2(\bm{x}')|\leq 2L_k\|\bm{x}-\bm{x}'\|\nonumber\\
	&+\left\|\bm{k}(\bm{x},\bm{X}_N)-\bm{k}(\bm{x}',\bm{X}_N)\right\|
	\left\|(\bm{K}(\bm{X}_N,\bm{X}_N)+\sigma_n^2\bm{I}_N)^{-1}\right\|
	\left\|\bm{k}(\bm{X}_N,\bm{x})+\bm{k}(\bm{X}_N,\bm{x}')\right\|.
	\label{eq:sigdiff}
	\end{align}
	On the one hand, we have
	\begin{align}
	\|\bm{k}(\bm{x},\bm{X}_N)-\bm{k}(\bm{x}',\bm{X}_N)\|\leq 
	\sqrt{N}L_k\|\bm{x}-\bm{x}'\|
	\label{eq:sigminus}
	\end{align}
	due to Lipschitz continuity of $k(\bm{x},\bm{x}')$. On the other hand 
	we have
	\begin{align}
	\|\bm{k}(\bm{x},\bm{X}_N)+\bm{k}(\bm{x}',\bm{X}_N)\|\leq 2\sqrt{N}
	\max\limits_{\bm{x},\bm{x}'\in\mathbb{X}}k(\bm{x},\bm{x}').
	\label{eq:sigplus}
	\end{align}
	The modulus of continuity $\omega_{\sigma_N}(\tau)$ follows from 
	substituting \eqref{eq:sigminus} and \eqref{eq:sigplus} in \eqref{eq:sigdiff} 
	and taking the square root of the resulting expression. Finally, we prove
	the probabilistic uniform error bound by exploiting the fact that for every 
	grid $\mathbb{X}_{\tau}$ with $|\mathbb{X}_{\tau}|$ grid points and 
	\begin{align}
	\max\limits_{\bm{x}\in\mathbb{X}} \min\limits_{\bm{x}'\in\mathbb{X}_{\tau}}
	\|\bm{x}-\bm{x}'\|\leq \tau
	\label{eq:gridconstant}
	\end{align}
	it holds with probability of at least
	$1-|\mathbb{X}_{\tau}|\mathrm{e}^{-\beta(\tau)/2}$ that \cite{Srinivas2012}
	\begin{align}
	|f(\bm{x})-\nu_{N}(\bm{x})|\leq \sqrt{\beta(\tau)}\sigma_{N}(\bm{x}) 
	\quad \forall\bm{x}\in \mathbb{X}_{\tau}.
	\end{align}
	Choose \mbox{$\beta(\tau)=2\log\left(\frac{|\mathbb{X}_{\tau}|}{\delta}\right)$}, 
	then
	\begin{align}
	|f(\bm{x})-\nu_{N}(\bm{x})|\leq \sqrt{\beta(\tau)}\sigma_{N}(\bm{x}) 
	\quad \forall\bm{x}\in \mathbb{X}_{\tau}
	\end{align}
	holds with probability of at least $1-\delta$. Due to continuity of 
	$f(\bm{x})$, $\nu_N(\bm{x})$ and $\sigma_N(\bm{x})$ we obtain
	\begin{align}
	\min\limits_{\bm{x}'\in\mathbb{X}_{\tau}}|f(\bm{x})-f(\bm{x}')|&\leq 
	\tau L_f\quad \forall \bm{x}\in\mathbb{X}\\
	\min\limits_{\bm{x}'\in\mathbb{X}_{\tau}}|\nu_N(\bm{x})-\nu_N(\bm{x}')|&\leq 
	\tau L_{\nu_N}\quad \forall \bm{x}\in\mathbb{X}\\
	\min\limits_{\bm{x}'\in\mathbb{X}_{\tau}}|\sigma_N(\bm{x})-\sigma_N(\bm{x}')|&\leq 
	\omega_{\sigma_N}(\tau)\quad \forall \bm{x}\in\mathbb{X}.
	\end{align}
	Moreover, the minimum number of grid points satisfying \eqref{eq:gridconstant} is 
	given by the covering number $M(\tau,\mathbb{X})$. Hence, we obtain
	\begin{align}
	P\left(|f(\bm{x})-\nu_{N}(\bm{x})|\leq 
	\sqrt{\beta(\tau)}\sigma_{N}(\bm{x})+\gamma(\tau), 
	~\forall\bm{x}\in\mathbb{X}\right)\geq 1-\delta,
	\end{align}
	where
	\begin{align}
	\beta(\tau)&=2\log\left(\frac{M(\tau,\mathbb{X})}{\delta}\right)\\
	\gamma(\tau)&=(L_f+L_{\nu_N})\tau+\sqrt{\beta(\tau)}\omega_{\sigma_N}(\tau).
	\end{align}
\end{proof}

\section{Proof of Theorem 3.2}
In order to proof Theorem 3.2, several auxiliary results are necessary,
which are derived in the following. The first lemma concerns the expected
supremum of a Gaussian process.
\begin{lemma}
	\label{lem:expsup}
	Consider a Gaussian process with a continuously differentiable covariance 
	function $k(\cdot,\cdot)$ and let $L_k$ denote its Lipschitz constant on 
	the set $\mathbb{X}$ with maximum extension  
	\mbox{$r=\max_{\bm{x},\bm{x}'\in\mathbb{X}}\|\bm{x}-\bm{x}'\|$}. 
	Then, the expected supremum of a sample function $f(\bm{x})$ of this
	Gaussian process satisfies
	\begin{align}
	E\left[\sup\limits_{\bm{x}\in \mathbb{X}}f(\bm{x})\right]\leq 
	12\sqrt{6d}\max\left\{\max\limits_{\bm{x}\in\mathbb{X}}\sqrt{k(\bm{x},\bm{x})},\sqrt{rL_k}\right\}.
	\end{align}
\end{lemma}
\begin{proof}
	We prove this lemma by making use of the metric entropy criterion for the 
	sample continuity of some version of a Gaussian process \cite{Dudley1967}.
	This criterion allows to bound the expected supremum of a sample function $f(\bm{x})$
	by
	\begin{equation}
	\mathrm{E}\left[ \sup\limits_{\bm{x}\in\mathbb{X}}f(\bm{x}) \right]\leq \int\limits_0^{\max\limits_{\bm{x}\in\mathbb{X}}\sqrt{k(\bm{x},\bm{x})}}
	\sqrt{\log(N(\varrho,\mathbb{X}))}\mathrm{d}\varrho,
	\label{eq:metEntropy}
	\end{equation}
	where $N(\varrho,\mathbb{X})$ is the $\varrho$-packing number of $\mathbb{X}$ 
	with respect to the covariance pseudo-metric
	\begin{align}
	d_k(\bm{x},\bm{x}')=\sqrt{k(\bm{x},\bm{x})+k(\bm{x}',\bm{x}')-2k(\bm{x},\bm{x}')}.
	\end{align}
	Instead of bounding the $\varrho$-packing number, we bound the $\varrho/2$-covering 
	number, which is known to be an upper bound. The covering number can be easily bounded 
	by transforming the problem of covering $\mathbb{X}$ with respect to the pseudo-metric 
	$d_k(\cdot,\cdot)$ into a coverage problem in the original metric of $\mathbb{X}$. For 
	this reason, define
	\begin{align}
	\psi(\varrho')=\sup\limits_{\subalign{\bm{x},\bm{x}'&\in \mathbb{X}\\ 
			\|\bm{x}-\bm{x}'&\|_{\infty}\leq \varrho'}} d_k(\bm{x},\bm{x}'),
	\end{align}
	which is continuous due to the continuity of the covariance kernel $k(\cdot,\cdot)$. 
	Consider the inverse function
	\begin{align}
	\psi^{-1}(\varrho)=\inf\left\{\varrho'>0:~\psi(\varrho')>\varrho\right\}.
	\end{align}
	Continuity of $\psi(\cdot)$ implies $\varrho=\psi(\psi^{-1}(\varrho))$. In particular, 
	this means that we can guarantee $d_k(\bm{x},\bm{x}')\leq \frac{\varrho}{2}$ if 
	\mbox{$\|\bm{x}-\bm{x}'\|\leq \psi^{-1}(\frac{\varrho}{2})$}. Due to this relationship 
	it is sufficient to construct an uniform grid with grid constant $2\psi^{-1}(\frac{\varrho}{2})$ 
	in order to obtain a $\varrho/2$-covering net of $\mathbb{X}$. Furthermore, the cardinality 
	of this grid is an upper bound for the $\varrho/2$-covering number, i.e.
	\begin{align}
	M(\varrho/2,\mathbb{X})\leq 
	\left\lceil \frac{r}{2\psi^{-1}(\frac{\varrho}{2})} \right\rceil^{d}.
	\end{align}
	Therefore, it follows that
	\begin{align}
	N(\varrho,\mathbb{X})\leq 
	\left\lceil \frac{r}{2\psi^{-1}(\frac{\varrho}{2})} \right\rceil^{d}.
	\end{align}
	Due to the Lipschitz continuity of the covariance function, we can bound 
	$\psi(\cdot)$ by
	\begin{align}
	\psi(\varrho')&\leq \sqrt{2L_k\varrho'}.
	\end{align}
	Hence, the inverse function satisfies
	\begin{align}
	\psi^{-1}\left(\frac{\varrho}{2}\right)\geq 
	\left(\frac{\varrho}{2\sqrt{2L_k}}\right)^2
	\end{align}
	and consequently
	\begin{align}
	N(\varrho,\mathbb{X})\leq \left(1+\frac{4rL_k}{\varrho^2}\right)^{d}
	\end{align}
	holds, where the ceil operator is resolved through the addition of $1$.
	Substituting this expression in the metric entropy bound 
	\eqref{eq:metEntropy} yields
	\begin{align}
	E\left[\sup\limits_{\bm{x}\in \mathbb{X}}f(\bm{x})\right]\leq 12\sqrt{d}
	\int\limits_0^{\max\limits_{\bm{x}\in\mathbb{X}}\sqrt{k(\bm{x},\bm{x})}}
	\sqrt{\log\left(1+\frac{4rL_k}{\varrho^2}\right)}\mathrm{d}\varrho.
	\end{align}
	As shown in \cite{Grunewalder2010} this integral can be bounded by
	\begin{align}
	\int\limits_0^{\max\limits_{\bm{x}\in\mathbb{X}}\sqrt{k(\bm{x},\bm{x})}}
	\sqrt{\log\left(1+\frac{4rL_k}{\varrho^2}\right)}
	\mathrm{d}\varrho\leq
	\sqrt{6}\max\left\{\max\limits_{\bm{x}\in\mathbb{X}}\sqrt{k(\bm{x},\bm{x})},
	\sqrt{rL_k}\right\}
	\end{align}
	which proves the lemma.
\end{proof}
Based on the expected supremum of Gaussian process it is possible to 
derive a high probability bound for the supremum of a sample function.
\begin{lemma}
	\label{lem:supbound}
	Consider a Gaussian process with a continuously differentiable covariance 
	function $k(\cdot,\cdot)$ and let $L_k$ denote its Lipschitz constant on 
	the set $\mathbb{X}$ with maximum extension  
	\mbox{$r=\max_{\bm{x},\bm{x}'\in\mathbb{X}}\|\bm{x}-\bm{x}'\|$}. 
	Then, with probability of at least $1-\delta_L$ the supremum of a sample
	function $f(\bm{x})$ of this Gaussian process is bounded by 
	\begin{align}
	\sup\limits_{\bm{x}\in\mathbb{X}}f(\bm{x})\leq &
	\sqrt{2\log\left( \frac{1}{\delta_L} \right)}\max\limits_{\bm{x}\in\mathbb{X}}
	\sqrt{k(\bm{x},\bm{x})}+12\sqrt{6d}
	\max\left\{\max\limits_{\bm{x}\in\mathbb{X}}\sqrt{k(\bm{x},\bm{x})},
	\sqrt{rL_{k}}\right\}.
	\end{align}
\end{lemma}
\begin{proof}
	We prove this lemma by exploiting the wide theory of concentration inequalities 
	to derive a bound for the supremum of the sample function $f(\bm{x})$. We apply the 
	Borell-TIS inequality \cite{Talagrand1994}
	\begin{align}
	P\Bigg( \sup\limits_{\bm{x}\in\mathbb{X}}f(\bm{x})-
	E\Bigg[ \sup\limits_{\bm{x}\in\mathbb{X}}f&(\bm{x}) \Bigg] \geq 
	c \Bigg)\leq \exp\left( -\frac{c^2}{2\max\limits_{\bm{x}\in\mathbb{X}}
		k(\bm{x},\bm{x})} \right).
	\label{eq:Tal}
	\end{align}
	Due to \cref{lem:expsup} we have
	\begin{align}
	\label{eq:Esup}
	E\left[\sup\limits_{\bm{x}\in \mathbb{X}}f(\bm{x})\right]\leq 
	12\sqrt{6d}\max\left\{\max\limits_{\bm{x}\in\mathbb{X}}\sqrt{k(\bm{x},\bm{x})},\sqrt{rL_k}\right\}.
	\end{align}
	The lemma follows from 
	substituting \eqref{eq:Esup} in \eqref{eq:Tal} and choosing 
	$c=\sqrt{2\log\left( \frac{1}{\delta_L} \right)}\max\limits_{\bm{x}\in\mathbb{X}}\sqrt{k(\bm{x},\bm{x})}$.
\end{proof}
Finally, we exploit the fact that the derivative of a sample function is a sample function
from another Gaussian process to prove the high probability Lipschitz constant in Theorem 3.2.
\begin{proof}[Proof of Theorem 3.2]
	Continuity of the sample function $f(\bm{x})$ follows directly from 
	\cite[Theorem 5]{Ghosal2006}. Furthermore, this theorem guarantees that the derivative 
	functions $\frac{\partial}{\partial x_i}f(\bm{x})$ are samples from derivative Gaussian 
	processes with covariance functions
	\begin{align}
	k_{\partial i}(\bm{x},\bm{x}')=
	\frac{\partial^2}{\partial x_i\partial x_i'}k(\bm{x},\bm{x}').
	\end{align}
	Therefore, we can apply \cref{lem:supbound} to each of the derivative processes and 
	obtain with probability of at least $1-\frac{\delta_L}{d}$
	\begin{align}
	-L_{f_{\partial i}} \leq\sup\limits_{\bm{x}\in\mathbb{X}}\frac{\partial}{\partial x_i}f(\bm{x}) 
	\leq L_{f_{\partial i}}, 
	\label{eq:partderbound}
	\end{align}
	where
	\begin{align}
	L_{f_{\partial i}}&=\sqrt{2\log\left( \frac{2d}{\delta_L} \right)}\max\limits_{\bm{x}\in\mathbb{X}}
	\sqrt{k_{\partial i}(\bm{x},\bm{x})}+12\sqrt{6d}
	\max\left\{\max\limits_{\bm{x}\in\mathbb{X}}\sqrt{k_{\partial i}(\bm{x},\bm{x})},
	\sqrt{rL_{k}^{\partial i}}\right\}
	\end{align}
	and $L_{k}^{\partial i}$ is the Lipschitz constant of derivative kernel 
	$k_{\partial i}(\bm{x},\bm{x}')$. Applying the union bound over all partial derivative 
	processes $i=1,\ldots,d$ finally yields the result.
\end{proof}

\section{Proof of Theorem 3.3}

\begin{proof}[Proof of Theorem 3.3]
	Due to Theorem 3.1 with $\beta_N(\tau)=2\log\left( \frac{M(\tau,\mathbb{X})\pi^2N^2}{3\delta} \right)$
	and the union bound over all $N>0$ it follows that
	\begin{align}
	\sup\limits_{\bm{x}\in \mathbb{X}}|f(\bm{x})-\nu_{N}(\bm{x})|\leq
	\sqrt{\beta_N(\tau)}\sigma_{N}(\bm{x})+\gamma_N(\tau)\quad \forall N>0
	\label{eq:un_err}
	\end{align}
	with probability of at least $1-\delta/2$. A trivial 
	bound for the covering number can be obtained by considering a uniform grid over the 
	cube containing $\mathbb{X}$. This approach leads to
	\begin{align}
	M(\tau,\mathbb{X})\leq \left(1+\frac{r}{\tau}\right)^{d},
	\end{align}
	where \mbox{$r=\max_{\bm{x},\bm{x}'\in\mathbb{X}}\|\bm{x}-\bm{x}'\|$}. Therefore,
	we have
	\begin{align}
	\beta_N(\tau)\leq 2d\log\left(1+\frac{r}{\tau}\right)+4\log(\pi N)-2\log(3\delta).
	\label{eq:beta(tau)}
	\end{align}
	Furthermore, the Lipschitz constant $L_{\nu_N}$ is bounded by 
	\begin{align}
	L_{\nu_N}&\leq L_k\sqrt{N} 
	\left\| (\bm{K}(\bm{X}_N,\bm{X}_N)+\sigma_n^2\bm{I}_N)^{-1}\bm{y}_N \right\|
	\end{align}
	due to Theorem 3.1. Since the Gram matrix $\bm{K}(\bm{X}_N,\bm{X}_N)$ is positive 
	semidefinite and $f(\cdot)$ is bounded by $\bar{f}$, we can bound 
	$\left\| (\bm{K}(\bm{X}_N,\bm{X}_N)+\sigma_n^2\bm{I}_N)^{-1}\bm{y}_N \right\|$ by
	\begin{align}
	\left\| (\bm{K}(\bm{X}_N,\bm{X}_N)+\sigma_n^2\bm{I}_N)^{-1}\bm{y}_N \right\|&
	\leq\frac{\|\bm{y}_N\|}
	{\rho_{\min}(\bm{K}(\bm{X}_N,\bm{X}_N)+\sigma_n^2\bm{I}_N)}\nonumber\\
	&\leq \frac{\sqrt{N}\bar{f}
		+\|\bm{\xi}_N\|}{\sigma_n^2},
	\end{align}
	where $\bm{\xi}_N$ is a vector of $N$ i.i.d. zero mean Gaussian random variables with
	variance $\sigma_n^2$. Therefore, it follows that 
	$\frac{\|\bm{\xi}_N\|^2}{\sigma_n^2}\sim\chi_N^2$. Due to 
	\cite{Laurent2000}, with probability of at least $1-\exp(-\eta_N)$ we have 
	\begin{align}
	\|\bm{\xi}_N\|^2\leq \left(2\sqrt{N\eta_N}+2\eta_N+N\right)\sigma_n^2.
	\end{align}
	Setting $\eta_N=\log(\frac{\pi^2N^2}{3\delta})$ and applying the union bounds 
	over all $N>0$ yields
	\begin{align}
	\left\| (\bm{K}(\bm{X}_N,\bm{X}_N)+\sigma_n^2\bm{I}_N)^{-1}\bm{y}_N \right\|\leq
	\frac{\sqrt{N}\bar{f}+\sqrt{2\sqrt{N\eta_N}+2\eta_N+N}\sigma_n}
	{\sigma_n^2}\quad \forall N>0
	\end{align}
	with probability of at least $1-\delta/2$. Hence, the Lipschitz constant of the 
	posterior mean function $\nu_N(\cdot)$ satisfies with probability of at least 
	$1-\delta/2$
	\begin{align}
	L_{\nu_N}\leq L_k\frac{N\bar{f}+\sqrt{N(2\sqrt{N\eta_N}+2\eta_N+N)}\sigma_n}
	{\sigma_n^2}\quad \forall N>0.
	\end{align}
	Since $\eta_N$ grows logarithmically with the number of training samples 
	$N$, it holds that $L_{\nu_N}\in\mathcal{O}(N)$ with 
	probability of at least $1-\delta/2$. The modulus of continuity 
	$\omega_{\sigma_N}(\cdot)$ of the posterior standard deviation can be bounded by
	\begin{align}
	\omega_{\sigma_N}(\tau)\leq\sqrt{2L_k\tau \left( \frac{N\max\limits_{
				\tilde{\bm{x}},\tilde{\bm{x}}'\in\mathbb{X}}k(\tilde{\bm{x}},
			\tilde{\bm{x}}')}{\sigma_n^2}+1\right)}
	\end{align}
	because $\|(\bm{K}(\bm{X}_N,\bm{X}_N)+\sigma_n^2\bm{I}_N)^{-1}\|\leq 
	\frac{1}{\sigma_n^2}$. Due to the union bound \eqref{eq:un_err} holds with 
	probability of at least $1-\delta$ with 
	\begin{align}
	\gamma_N(\tau)\leq&\sqrt{2L_k\tau\beta(\tau)\left( \frac{N\max\limits_{\tilde{\bm{x}},\tilde{\bm{x}}'
				\in\mathbb{X}}k(\tilde{\bm{x}},\tilde{\bm{x}}')}{\sigma_n^2}\!+\!1 \right) }
	\!+\!L_f\tau\!+\!L_k\frac{N\bar{f}\!+\!\sqrt{N(2\sqrt{N\eta_N}\!+\!2\eta_N\!+\!N)}}{\sigma_n^2}\tau.
	\end{align}
	This function must converge to $0$ for $N\rightarrow\infty$ in order to guarantee 
	a vanishing regression error. This is only ensured if $\tau(N)$ decreases faster 
	than $\mathcal{O}((N\log(N))^{-1})$. Therefore, set $\tau(N)\in\mathcal{O}(N^{-2})$ in order 
	to guarantee
	\begin{align}
	\lim\limits_{N\rightarrow\infty}\gamma_N(\tau_N)=0.
	\end{align}
	However, this choice of $\tau(N)$ implies that $\beta_N(\tau(N))\in\mathcal{O}(\log(N))$ 
	due to \eqref{eq:beta(tau)}. Since there exists an $\epsilon>0$ such that 
	$\sigma_N(\bm{x})\in\mathcal{O}\left(\log(N)^{-\frac{1}{2}-\epsilon}\right)$, 
	$\forall\bm{x}\in\mathbb{X}$ by assumption, we have
	\begin{align}
	\sqrt{\beta_N(\tau(N))}\sigma_N(\bm{x})\in\mathcal{O}(\log(N)^{-\epsilon})
	\quad\forall\bm{x}\in\mathbb{X},
	\end{align}
	which concludes the proof.
\end{proof}

\section{Proof of Theorem 4.1}
Lyapunov theory provides the following statement~\cite{Khalil2002}.
\begin{lemma}
	\label{lem:Lyap}
	A dynamical system~$\dot{\x} = \bm{f}(\x,\bm{u})$ is globally ultimately 
	bounded to a set~$\Bset\subset\Xset$, containing the origin, if there 
	exists a 
	positive definite (so called Lyapunov) function, $V:\Xset \to \Rset_{+,0}$, 
	for 
	which~$\dot{V}(\x)<0$, for all~$\x \in \Xset \setminus \Bset$.
\end{lemma}
This allows to proof Theorem 4.1 as following.
\begin{proof}[Proof of Theorem 4.1]
	Consider the Lyapunov function~$V(\x)=\frac{1}{2} r^2$ 
	\begin{align*}
	\dot{V}(\x) &= \frac{\partial V}{\partial r} \dot{r} 
	= r\left(f(\x)-\fh(\x) -\kc r\right)
	\leq \abs{r} \abs*{f(\x)-\nu_{N}(\x)}  -\kc \abs{r}^2
	\leq 0 \quad  \forall \abs{r}>\frac{f(\x)-\nu_{N}(\x)}{\kc}
	\end{align*}
	Based on Theorem 3.1, the model error is bounded with high 
	probability, which allows to conclude
	\begin{align*}
	P\left(\dot{V}(\x)<0 ~\forall\x\in\Xset\setminus \Bset\right)\geq 
	1-\delta.
	\end{align*}
	The global ultimate boundedness of the closed-loop system, is thereby shown 
	according to \cref{lem:Lyap}.
\end{proof}

\section{Report on Computational Complexity of the Numerical Evaluation}
Simulations are performed in MATLAB 2019a on a i5-6200U CPU with 2.3GHz 
and 8GB RAM. The simulation in Sec. 5.1 took 77s and used 1 MB of workspace 
memory. 
The simulation in Sec. 5.2 took 39s and used 134 MB of workspace memory. The 
code is available as supplementary material.

\end{document}